\documentclass{article}

\usepackage[numbers, sort&compress]{natbib}
\usepackage[final]{nips_2016}

\usepackage[utf8]{inputenc} 
\usepackage[T1]{fontenc}    
\usepackage{url}            
\usepackage{booktabs, tabularx}       
\usepackage{amsfonts}       
\usepackage{nicefrac}       
\usepackage[USenglish,british,american,australian,english]{babel}
\usepackage{amsthm,amsmath,amssymb}
\usepackage{multirow}
\usepackage[pdftex]{graphicx}
\usepackage{caption}
\usepackage{subcaption}
\usepackage{float}
\usepackage{makecell}
\usepackage{booktabs}
\usepackage{array}
\usepackage{url}
\usepackage[algo2e]{algorithm2e}
\usepackage{bm}
\usepackage{bbm}
\usepackage{smile}
\usepackage{lipsum}
\usepackage{mathrsfs}
\usepackage{dsfont}
\usepackage{epstopdf}
\usepackage{caption}
\usepackage{marvosym}
\usepackage{diagbox}

\usepackage{listings}
\usepackage[usenames,dvipsnames,svgnames,table]{xcolor}
\usepackage[colorlinks=true,
            linkcolor=red,
            urlcolor=blue,
            citecolor=blue]{hyperref}
\usepackage{asymptote}

\usepackage[utf8]{inputenc}
\usepackage{tkz-berge}
\definecolor{fondpaille}{cmyk}{0,0,0.1,0}

\def\pb{}
\def\b{\beta}
\def\dcv{\stackrel{\scriptstyle d}{\longrightarrow}} 
\def\mn{{}}
\def\beq{\begin{equation} }
\def\eeq{\end{equation} }
\def\ep{\varepsilon}
\def\1{\mathbf{1}}

\numberwithin{equation}{section}

\def\be{\mathbf{e}}
\def\ub{\mathbf{u}}
\def\vb{\mathbf{v}}

\def\[{\lf[} \def\]{\ri]}   \def\lf{\left} \def\ri{\right}

\def\old#1{}

\def\to{\rightarrow}

\def\E{\mathbb{E}}
\def\P{\mathbb{P}}
 
\def\argmin{{\rm argmin}}
 \def\O{\Omega}

\def\b{\beta}

\def\F{\mathcal{F}}

\def\O{\mathcal{O}}

\def\E{\mathbb{E}}
\def\P{\mathbb{P}}
\def\mn{\medskip\noindent}
\usepackage[compact]{titlesec}
\titlespacing*{\section}
{0pt}{1ex}{1ex}
\titlespacing*{\subsection}
{0pt}{1ex}{1ex}
\titlespacing*{\subsubsection}
{0pt}{1ex}{1ex}
\usepackage{enumitem} 
 \setlist{nolistsep}
  \nopagebreak[4]
\usepackage[activate={true,nocompatibility},final,tracking=true,kerning=true,spacing=true]{microtype}

\def\beq{\begin{equation} }
\def\eeq{\end{equation} }
\newtheorem{theo}{Theorem}
\newtheorem{lemm}{Lemma}
\newtheorem{prop}{Proposition}

\newtheorem{assu}{Assumption}

\newtheorem*{rema}{Remark}

\title{Online ICA: Understanding Global Dynamics of Nonconvex Optimization via Diffusion Processes}
\date{}
\author{
Chris Junchi Li\qquad Zhaoran Wang\qquad Han Liu
\\
Department of Operations Research and Financial Engineering, Princeton University
\\
{\{junchil, zhaoran, hanliu\}@princeton.edu}
}

\begin{document}

\maketitle

\setlength\abovedisplayskip{3pt}
\setlength\belowdisplayskip{3pt}

\begin{abstract}
Solving statistical learning problems often involves nonconvex optimization. Despite the empirical success of nonconvex statistical optimization methods, their global dynamics, especially convergence to the desirable local minima, remain less well understood in theory. In this paper, we propose a new analytic paradigm based on diffusion processes to characterize the global dynamics of nonconvex statistical optimization. As a concrete example, we study stochastic gradient descent (SGD) for the tensor decomposition formulation of independent component analysis. In particular, we cast different phases of SGD into diffusion processes, i.e., solutions to stochastic differential equations. Initialized from an unstable equilibrium, the global dynamics of SGD transit over three consecutive phases: (i) an unstable Ornstein-Uhlenbeck process slowly departing from the initialization, (ii) the solution to an ordinary differential equation, which quickly evolves towards the desirable local minimum, and (iii) a stable Ornstein-Uhlenbeck process oscillating around the desirable local minimum. Our proof techniques are based upon Stroock and Varadhan's weak convergence of Markov chains to diffusion processes, which are of independent interest.
\end{abstract}

\section{Introduction}\label{sec:intro}
 For solving a broad range of large-scale statistical learning problems, e.g., deep learning, nonconvex optimization
 methods often exhibit favorable computational and statistical efficiency empirically. However, there is still a
 lack of  theoretical understanding of the global dynamics of these nonconvex optimization methods. In specific, it
 remains largely unexplored why simple optimization algorithms, e.g., stochastic gradient descent (SGD), often exhibit fast convergence towards local minima with desirable statistical accuracy. In this paper, we aim to develop a new analytic framework to theoretically understand this phenomenon.


The dynamics of nonconvex statistical optimization are of central interest to a recent line of work.
Specifically, by exploring the local convexity within the basins of attraction, \cite{jain2013low, netrapalli2013phase, agarwal2013learning, anandkumar2014analyzing,
   wang2014optimal, wang2014nonconvex, netrapalli2014non, hardt2014understanding,
   jain2014fast, candes2014phase, balakrishnan2014statistical, zhang2014spectral, qu2014finding, gu2014sparse, sun2015provable, loh15regularized, sun2015guaranteed, arora2015simple, zhao2015nonconvex, chen2015solving, tu2015low,
   chen2015fast, bhojanapalli2015dropping, white2015local, zheng2015convergent, cai2015optimal, jain2015computing,
   wang2015high, yang2015sparse, sun2016sparse, gu2016low, tan2016sparse} establish local fast rates of convergence towards the desirable local
   minima for a variety statistical problems. Most of these characterizations of local dynamics are based on two
   decoupled ingredients from statistics and optimization: (i) the local (approximately) convex geometry of the
   objective functions, which is induced by the underlying statistical models, and (ii) adaptation of classical
   optimization analysis \cite{nesterov2003introductory, golub2012matrix} by incorporating the perturbations induced
   by nonconvex geometry as well as random noise. To achieve global convergence guarantees, they rely on various problem-specific
   approaches to obtain initializations that provably fall into the basins of attraction. Meanwhile, for some learning
   problems, such as phase retrieval and tensor decomposition for latent variable models, it is empirically
   observed that good initializations within the basins of attraction are not essential to the desirable convergence.
   However, it remains highly challenging to characterize the global dynamics, especially within the highly nonconvex
   regions outside the local basins of attraction. 

   In this paper, we address this problem with a new analytic framework based on diffusion processes. In particular, 
   we focus on the concrete example of SGD applied on the tensor decomposition formulation of independent component
   analysis (ICA). Instead of adapting classical optimization analysis accordingly to local nonconvex geometry,
   we cast SGD in different phases as diffusion processes, i.e., solutions to stochastic differential equations (SDE), by
   analyzing the weak convergence from discrete Markov chains to their continuous-time limits \cite{stroock1979multidimensional, EthierKurtz}.
The SDE automatically incorporates the geometry and randomness induced by the statistical model, which allows us to establish the exact dynamics of SGD. In contrast, classical optimization analysis only yields upper bounds on the optimization error, which are unlikely to be tight in the presence of highly nonconvex geometry, especially around
   the stationary points that have negative curvatures along certain directions. In particular, we identify three consecutive phases of the
   global dynamics of SGD, which is illustrated in Figure \ref{fig:illu}. 
   \begin{enumerate}[label=(\roman*)]
   \item We consider the most challenging initialization at a stationary point with negative curvatures, which can
   be cast as an unstable equilibrium of the SDE. Within the first phase, the dynamics of SGD are characterized by an
   unstable Ornstein-Uhlenbeck process \cite{OKSENDAL,AldousPoisson}, which departs from the initialization at a relatively slow rate and enters the
   second phase. 
   \item Within the second phase, the dynamics of SGD are characterized by the exact solution to an ordinary
   differential equation. This solution evolves towards the desirable local minimum at a relatively fast rate until it approaches a small basin around the local minimum. 
   \item Within the third phase, the dynamics of SGD are captured by a stable Ornstein-Uhlenbeck process \cite{OKSENDAL,AldousPoisson}, which
   oscillates within a small basin around the local minimum. 
   \end{enumerate}

\vskip15pt
   \begin{figure*}[!htb]
   \centering
   \includegraphics[width=0.9\textwidth]{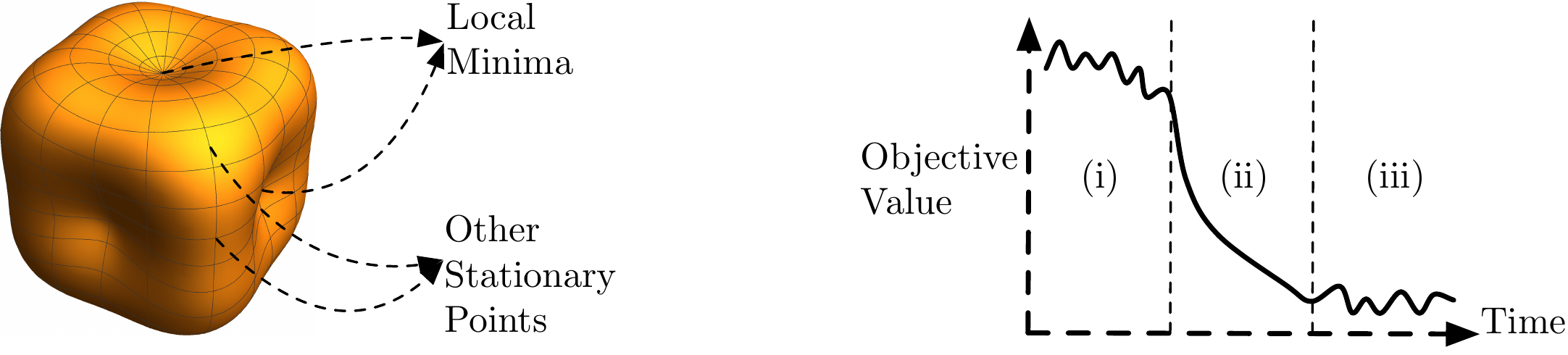}
   \caption{Left: an illustration of the objective function for the tensor decomposition formulation of ICA. Note that here we use the spherical coordinate system and add a global offset of $2$ to the objective function for better illustration. Right: An illustration of the three phases of diffusion processes.}
   \label{fig:illu}
   \end{figure*}
\vskip8pt

   \noindent\textbf{More related work.} 
   Our results are connected with a very recent line of work 
    \cite{ge2015escaping, sun2015complete, sun2015complete2, sun2015nonconvex, sun2016geometric, lee2016gradient, li2016near,
    anandkumar2016efficient, panageas2016gradient} on the global dynamics of nonconvex statistical optimization. In
    detail, they characterize the global geometry of nonconvex objective functions, especially around their
    saddle points or local maxima. Based on the geometry, they prove that specific optimization algorithms, e.g.,
    SGD with artificial noise injection, gradient descent with random initialization, and second-order
    methods, avoid the saddle points or local maxima, and globally converge to the desirable local minima. Among these
    results, our results are most related to \cite{ge2015escaping}, which considers SGD with noise injection on
    ICA. Compared with this line of work, our analysis takes a completely different approach based on diffusion
    processes, which is also related to another line of work
      \cite{darken1991towards, de2014global, su2014differential, li2015dynamics, mobahi2016training, mandt2016variational}.

    Without characterizing the global geometry, we establish the global exact dynamics of SGD, which
    illustrate that, even starting from the most challenging stationary point, it may be unnecessary to use additional
    techniques such as noise injection, random initialization, and second-order information to ensure the desirable
    convergence. In other words, the unstable Ornstein-Uhlenbeck process within the first phase itself is powerful
    enough to escape from stationary points with negative curvatures. This phenomenon is not captured by the
    previous upper bound-based analysis, since previous upper bounds are relatively coarse-grained compared with the exact
    dynamics, which naturally give a sharp characterization simultaneously from upper and lower bounds. Furthermore,
    in Section \ref{sec:phase} we will show that our sharp diffusion process-based characterization provides understanding on different phases of dynamics of our online/SGD algorithm for ICA. 

%
          A recent work \cite{li2016near} analyzes an online principal component analysis algorithm based on the
          intuition gained from diffusion approximation. In this paper, we consider a different statistical problem with a rigorous
          characterization of the diffusion approximations in three separate phases. 

                \noindent\textbf{Our contribution.} In summary, we propose a new analytic paradigm based on diffusion
                processes for characterizing the global dynamics of nonconvex statistical optimization. For SGD on
                ICA, we identify the aforementioned three phases for the first time. Our analysis is based on
                Stroock and Varadhan's weak convergence of Markov chains to diffusion processes, which are of
                independent interest.


\section{Background}\label{sec:setting}

In this section we formally introduce a special model of \emph{independent component analysis} (ICA) and the associated SGD algorithm. Let $\{\bX^{(i)}\}_{i=1}^n$ be the data sample identically distributed as $\bX\in \RR^{d}$. We make assumptions for the distribution of $\bX$ as follows. Let $\|\cdot\|$ be the $\ell_2$-norm of a vector. 
\begin{assu}\label{assu:distribution}
There is an orthonormal matrix $\Ab\in \RR^{d\times d}$ such that $\bX = \Ab \bY$, where $\bY \in \RR^{d}$ is a random vector that has independent entries satisfying the following conditions: 
\begin{enumerate}[label=(\roman*), topsep=0pt, leftmargin=*]
\item
The distribution of each $Y_i$ is symmetric about 0;
\item
There is a constant $B$ such that $\|\bY\|^2 \le B$;
\item
The $Y_1, \ldots, Y_d$ are independent with identical $m$ moments for $m\le 8$, denoted by $\psi_m \equiv \E Y_1^m$;
\item
The $\psi_1 = \E Y_i = 0$, $\psi_2 = \E Y_i^2 = 1$, $\psi \equiv \psi_4 \ne 3$.
\end{enumerate} 
\end{assu}


Assumption \ref{assu:distribution}(iii) above is a generalization of i.i.d.~tensor components. Let $\Ab = \left(\ab_1,\ldots, \ab_d \right)$ whose columns form an orthonormal basis. Our goal is to estimate the orthonormal basis $\ab_i$ from online data $\bX_1, \ldots, \bX_n$. We first establish a preliminary lemma.
\begin{lemm}\label{lemm:vTY}
Let $\Tb \!=\! \EE(\bX^{\otimes 4})$ be the 4th-order tensor whose $(i,j,k,l)$-entry is $\E\left( X_iX_jX_kX_l\right)$. Under Assumption \ref{assu:distribution}, we have
\beq\label{EuTX}
\Tb(\ub,\ub,\ub,\ub) \equiv \E\left( \ub^\top \bX \right)^4 = 3 +(\psi-3) \sum_{i=1}^d (\ab_i^\top \ub)^4.
\eeq
\end{lemm}
Lemma \ref{lemm:vTY} implies that finding $\ab_i$'s can be cast into the solution to the following population optimization problem 
\beq\label{opt1}
\argmin  -\sign(\psi-3) \cdot \E\left( \ub^\top \bX \right)^4 
= \argmin \sum_{i=1}^d - (\ab_i^\top \ub)^4~~~~~
\mbox{subject to }\|\ub\| = 1.
\eeq

It is straightforward to conclude that all stable equilibria of \eqref{opt1} are $\pm \ab_i$ whose number linearly grows with $d$. Meanwhile, by analyzing the Hessian matrices the set of unstable equilibria of \eqref{opt1} includes (but not limited to) all $\vb^* = d^{-1/2} (\pm 1,\cdots, \pm 1)$, whose number grows exponentially as $d$ increases \cite{ge2015escaping,sun2015nonconvex}. 

Now we introduce the SGD algorithm for solving \eqref{opt1} with finite samples. Let $\cS^{d-1}= \{\ub: \|\ub\| = 1\}$ be the unit sphere in $\RR^d$, and denote $\Pi \ub = \ub / \|\ub\|$ for $\ub\ne 0$ the projection operator onto $\cS^{d-1}$. With appropriate initialization, the SGD for tensor method iteratively updates the estimator via the following Eq.~\eqref{tensoralgo}:
\beq\label{tensoralgo}
\ub^{(n)} = \Pi \left\{ \ub^{(n-1)} + \sign(\psi-3)\cdot \beta \left(\ub^{(n-1)}\,^\top \bX^{(n)} \right)^3 \bX^{(n)} \right\}.
\eeq
The SGD algorithms that performs stochastic approximation using single online data sample in each update has the advantage of less temporal and spatial complexity, especially when $d$ is high \cite{li2016near,ge2015escaping}. An essential issue of this nonconvex optimization problem is how the algorithm escape from unstable equilibria. \cite{ge2015escaping} provides a method of adding artificial noises to the samples, where the noise variables are uniformly sampled from $\cS^{d-1}$. In our work, we demonstrate that under some reasonable distributional assumptions, the online data provide sufficient noise for the algorithm to escape from the unstable equilibria.

By symmetry, our algorithm in Eq.~\eqref{tensoralgo} converges to a uniformly random tensor component from $d$ components. In order to solve the problem completely, one can repeatedly run the algorithm using different set of online samples until all tensor components are found. In the case where $d$ is high, the well-known {coupon collector problem} \cite{DurrettPTE} implies that it takes $\approx d\log d$ runs of SGD algorithm to obtain all $d$ tensor components.

\begin{rema}
From Eq.~\eqref{opt1} we see the tensor structure in Eq.~\eqref{EuTX} is unidentifiable in the case of $\psi = 3$, see more discussion in \cite{anandkumar2014tensor,ge2015escaping}. Therefore in Assumption \ref{assu:distribution} we rule out the value $\psi = 3$ and call the value $\left|\psi-3\right|$ the \emph{tensor gap}. The reader will see later that, analogous to eigengap in SGD algorithm for principal component analysis (PCA) \cite{li2016near}, tensor gap plays a vital role in the time complexity in the algorithm analysis.
\end{rema}


\pb\section{Markov Processes and Differential Equation Approximation}\label{sec:infquan}
To work on the approximation we first conclude the following proposition. 
\begin{prop}\label{prop:Markov}
The iteration $\ub^{(n)}, n = 0, 1, \ldots$ generated by Eq.~\eqref{tensoralgo} forms a discrete-time, time-homogeneous Markov process that takes values on $\cS^{d-1}$. Furthermore, $\ub^{(n)}$ holds strong Markov property.
\end{prop}
For convenience of analysis we use the transformed iteration $\vb^{(n)} \equiv \Ab^\top \ub^{(n)}$ in the rest of this paper. The update equation in Eq.~\eqref{tensoralgo} is equivalently written as
\beq\label{tensoralgo_update}
\begin{split}
\vb^{(n)} = \Ab^\top \ub^{(n)}
&= \Pi \left\{ \Ab^\top \ub^{(n-1)} \pm \beta \left(\ub^{(n-1)}\,^\top  \Ab \Ab^\top  \bX^{(n)} \right)^3 \Ab^\top \bX^{(n)} \right\} \\
&= \Pi \left\{ \vb^{(n-1)} \pm \beta \left(\vb^{(n-1)}\,^\top \bY^{(n)} \right)^3 \bY^{(n)} \right\}.
\end{split}
\eeq
Here $\pm \beta$ has the same sign with $\psi-3$. It is obvious from Proposition \ref{prop:Markov} that the (strong) Markov property applies to $\vb^{(n)}$, and one can analyze the iterates $\vb^{(n)}$ generated by Eq.~\eqref{tensoralgo_update} from a perspective of Markov processes.

Our next step is to conclude that as the stepsize $\beta\to 0^+$, the iterates generated by Eq.~\eqref{tensoralgo}, under the time scaling that speeds up the algorithm by a factor $\beta^{-1}$, can be globally approximated by the solution to the following ODE system. To characterize such approximation we use theory of weak convergence to diffusions \cite{stroock1979multidimensional,EthierKurtz} via computing the infinitesimal mean and variance for SGD for the tensor method. We remind the readers of the definition of weak convergence $Z^\beta \Rightarrow Z$ in stochastic processes: for any $0\le t_1<t_2<\cdots < t_n$ the following convergence in distribution occurs as $\b\to 0^+$
$$
\left( Z^\b(t_1), Z^\b(t_2), \ldots,Z^\b(t_n)\right)
\dcv \left(Z(t_1),Z(t_2), \ldots, Z(t_n)\right).
$$
To highlight the dependence on $\beta$ we add it in the superscipts of iterates $\vb^{\beta,(n)} = \vb^{(n)}$. Recall that $\lfloor t\beta^{-1} \rfloor$ is the integer part of the real number $t \beta^{-1}$.

\begin{theo}\label{theo:ODE}
If for each $ k = 1, \ldots, d$, as $\beta\to 0^+$ $v_k^{\beta, (0)}$ converges weakly to some constant scalar $V_k^o$ then the Markov process $v_k^{\beta, (\lfloor t\beta^{-1} \rfloor)}$ converges weakly to the solution of the ODE system
\beq\label{ODE}
\frac{\ud V_k}{\ud t} = \left|\psi-3\right| V_k\left( V_k^2 - \sum_{i=1}^d V_i^4 \right), \qquad k=1,\ldots,d,
\eeq
with initial values $V_k(0) = V_k^o$. 
\end{theo}

\mn
To understand the complex ODE system in Eq.~\eqref{ODE} we first investigate into the case of $d=2$. Consider a change of variable $V_1^2(t)$ we have by chain rule in calculus and $V_2^2 = 1-V_1^2$ the following derivation:
\begin{align}
\frac{\ud V_1^2}{\ud t} 
&= 
2V_1\cdot \frac{\ud V_1}{\ud t} 
= 
2V_1\cdot \left|\psi-3\right|V_1\left( V_1^2 - V_1^4 - V_2^4 \right) 
\nonumber
\\&=
2\left|\psi-3\right|V_1^2 \left( V_1^2 - V_1^4 - (1-V_1^2)^2 \right) 
= 
-2\left|\psi-3\right| V_1^2 \left(V_1^2 - \frac12\right) (V_1^2 - 1).
\label{ODE2d}
\end{align}
Eq.~\eqref{ODE2d} is an autonomous, first-order ODE for $V_1^2$. Although this equation is complex, a closed-form solution is available:
$$
V_1^2(t) = 0.5\pm 0.5(1+C\exp\left( -|\psi-3|t \right) )^{-0.5},
$$
and $V_2^2(t) = 1-V_1^2(t)$, where the choices of $\pm$ and $C$ depend on the initial value. The above solution allows us to conclude that if the initial vector $(V_1^o)^2 < (V_2^o)^2$ (resp.~$(V_1^o)^2 > (V_2^o)^2$), then it approaches to 1 (resp.~0) as $t\to\infty$. This intuition can be generalized to the case of higher $d$ that the ODE system in Eq.~\eqref{ODE} converges to the coordinate direction $\pm \be_k$ if $(V_k^o)^2$ is strictly maximal among $(V_1^o)^2, \ldots, (V_d^o)^2$ in the initial vector. To estimate the time of traverse we establish the following Proposition \ref{prop:ODEarrival}.

\begin{prop}\label{prop:ODEarrival}
Fix $\delta \in(0,1/2)$ and the initial value $V_k(0) = V_k^o$ that satisfies $(V_{k_0}^o)^2\ge 2 (V_k^o)^2$ for all $1\le k\le d, k\ne k_0$, then there is a constant (called traverse time) $T$ that depends only on $d,\delta$ such that $V_{k_0}^2(T)  \ge  1-\delta.$ Furthermore $T$ has the following upper bound: let $y(t)$ solution to the following auxillary ODE
\beq\label{ODEaux}
\frac{\ud y}{\ud t} = y^2 \left( 1-y \right),
\eeq
with $y(0) = 2/(d+1)$. Let $T_0$ be the time that $y(T_0) = 1-\delta$. Then
\beq\label{Tupper}
T \le |\psi-3|^{-1} T_0\le |\psi-3|^{-1} \left( d-3 + 4\log (2\delta)^{-1} \right).
\eeq
\end{prop}

Proposition \ref{prop:ODEarrival} concludes that, by admitting a gap of $2$ between the largest $(V_{k_0}^o)^2$ and second largest $(V_k^o)^2$, $k\ne k_0$ the estimate on traverse time can be given, which is tight enough for our purposes in Section \ref{sec:phase}.

\begin{rema}
In an earlier paper \cite{li2016near} which focuses on the SGD algorithm for PCA, when the stepsize is small, the algorithm iteration is approximated by the solution to ODE system after appropriate time rescaling. The approximate ODE system for SGD for PCA is
\beq\label{ODEpca}
\frac{\ud V_k}{\ud t} = -2V_k \sum_{i=1}^d (\lambda_k - \lambda_i) V_i^2, \qquad k=1,\ldots,d.
\eeq
The analysis there also involves computation of infinitesimal mean and variance for each coordinate as the stepsize $\beta\to 0^+$ and theory of convergence to diffusions \cite{stroock1979multidimensional,EthierKurtz}. A closed-form solution to Eq.~\eqref{ODEpca} is obtained in \cite{li2016near}, called the \emph{generalized logistic curves}. In contrast, to our best knowledge a closed-form solution to Eq.~\eqref{ODE} is generally \emph{not} available.
\end{rema}\smallskip

\pb\section{Local Approximation via Stochastic Differential Equation}\label{sec:SDE}
The ODE approximation in Section \ref{sec:infquan} is very informative: it characterizes globally the trajectory of our algorithm for ICA or tensor method in Eq.~\eqref{tensoralgo} with $\O(1)$ approximation errors. However it fails to characterize the behavior near equilibria where the gradients in our ODE system are close to zero. For instance, if the SGD algorithm starts from $\vb^*$, on a microscopic magnitude of $\O(\beta^{1/2})$ the noises generated by online samples help escaping from a neighborhood of $\vb^*$. 

Our main goal in this section is to demonstrate that under appropriate spatial and temporal scalings, the algorithm iteration converges locally to the solution to certain stochastic differential equations (SDE). We provide the SDE approximations in two scenarios, separately near an arbitrary tensor component (Subsection \ref{ssec:SDE1}) which indicates that our SGD for tensor method converges to a local minimum at a desirable rate, and a special local maximum (Subsection \ref{ssec:SDE2}) which implies that the stochastic nature of our SGD algorithm for tensor method helps escaping from unstable equilibria. Note that in the algorithm iterates, the escaping from stationary points occurs first, followed by the ODE and then by the phase of convergence to local minimum. We discuss this further in Section \ref{sec:phase}.

\pb\subsection{Neighborhood of Local Minimizers}\label{ssec:SDE1}
To analyze the behavior of SGD for tensor method we first consider the case where the iterates enter a neighborhood of one local minimizer, i.e.~the tensor component. Since the tensor decomposition in Eq.~\eqref{opt1} is full-rank and symmetric, we consider without loss of generality the neighborhood near $\eb_1$ the first tensor component. The following Theorem \ref{theo:SDE1} indicates that under appropriate spatial and temporal scalings, the process admits an approximation by Ornstein-Uhlenbeck process. Such approximation is characterized rigorously using weak convergence theory of Markov processes \cite{stroock1979multidimensional,EthierKurtz}.
The readers are referred to \cite{OKSENDAL} for fundamental topics on SDE.

\begin{theo}\label{theo:SDE1}
If for each $k = 2, \ldots, d$, $\b^{-1/2} v_k^{\beta,(0)}$ converges weakly to $U_k^o\in (0,\infty)$ as $\b\to 0^+$ then the stochastic process $\b^{-1/2} v_k^{\beta,(\lfloor t\b^{-1} \rfloor)}$ converges weakly to the solution of the stochastic differential equation
\beq\label{OUSDE1}
\ud U_k(t) = -\left|\psi-3\right| U_k(t)\ud t + \psi_6^{1/2}\ud B_k(t),
\eeq
with initial values $U_k(0) = U_k^o$. Here $B_k(t)$ is a standard one-dimensional Brownian motion. 
\end{theo}

We identify the solution to Eq.~\eqref{OUSDE1} as an Ornstein-Uhlenbeck process which can be expressed in terms of a It\^o integral, with
\beq\label{Ukt}
U_k(t) = U_k^o \exp\left( -|\psi-3| t \right) +  \psi_6^{1/2} \int_0^t \exp\left( -|\psi-3| (t-s) \right)  \ud B_k(s).
\eeq
It\^o isometry along with mean-zero property of It\^o integral gives
\begin{align*}
\E ( U_k(t))^2 
&=  (U_k^o)^2 \exp\left( -2|\psi-3| t \right) + \psi_6 \int_0^t \exp\left( -2|\psi-3| (t-s) \right)  \ud s \\
&= \frac{\psi_6}{ 2|\psi-3| } + \left( (U_k^o)^2 - \frac{\psi_6}{ 2|\psi-3| }\right) \exp\left( -2|\psi-3| t \right),
\end{align*}
which, by taking the limit $t\to\infty$, approaches $\psi_6 / (2|\psi-3| )$. From the above analysis we conclude that the Ornstein-Uhlenbeck process has the \emph{mean-reverting} property that its mean decays exponentially towards 0 with persistent fluctuations at equilibrium.

\pb\subsection{Escape from Unstable Equilibria}\label{ssec:SDE2}
In this subsection we consider SGD for tensor method that starts from a sufficiently small neighborhood of a special unstable equilibrium. We show that after appropriate rescalings of both time and space, the SGD for tensor iteration can be approximated by the solution to a second SDE. Analyzing the approximate SDE suggests that our SGD algorithm iterations can get rid of the unstable equilibria (including local maxima and stationary points with negative curvatures) whereas the traditional gradient descent (GD) method gets stuck. In other words, under weak distributional assumptions the stochastic gradient plays a vital role that helps the escape. As a illustrative example, we consider the special stationary points $\vb^* = d^{-1/2} (\pm 1, \ldots, \pm 1)$. Consider a submanifold $\cS_F \subseteq \cS^{d-1}$ where
$$
\cS_F = \left\{\vb\in \cS^{d-1}: \mbox{there exists $1\le k < k' \le d$ such that $v_k^2 = v_{k'}^2 = \max_{1\le i \le d} v_i^2$} \right\}.
$$
In words, $\cS_F$ consists of all $\vb\in \cS^{d-1}$ where the maximum of $v_k^2$ is \emph{not} unique. In the case of $d=3$, it is illustrated by Figure \ref{fig:illu} that $\cS_F$ is the frame of a 3-dimenisional box, and hence we call $\cS_F$ the \emph{frame}. Let
\beq\label{Wk}
W^\beta_{kk'}(t) = 
\b^{-1/2}\log\big( v_k^{\beta,(\lfloor t\b^{-1} \rfloor)}\big)^2
- \b^{-1/2}\log\big( v_{k'}^{\beta,(\lfloor t\b^{-1} \rfloor)}\big)^2.
\eeq
The reason we study $W^\beta_{kk'}(t)$ is that these $d(d-1)$ functions of $\vb\in \cS^{d-1}$ form a local coordinate map around $\vb^*$ and further characterize the distance between $\vb$ and $\cS_F$ on a spatial scale of $\beta^{1/2}$. We define the positive constant $\Lambda_{d,\psi}$ as
\beq\label{Lambda}
\begin{split}
\Lambda_{d,\psi}^2 
&=
8d^{-2}\left(
\psi_8 +  (16d-28) \psi_6  +  15d \psi_4^2 
\right.
\\&\quad
\left.
-5 (72 d^2-228 d+175)\psi_4
+ 15(2d-7)(d-2)(d-3)
\right).
\end{split}
\eeq
We have our second SDE approximation result as follows.
\begin{theo}\label{theo:SDE2}
Let $W^\beta_{kk'}(t)$ be defined as in Eq.~\eqref{Wk}, and let $\Lambda_{d,\psi}$ be as in Eq.~\eqref{Lambda}. If for each distinct $k,k' = 1, \ldots, d$, $W^\beta_{kk'}(0)$ converges weakly to $W_{kk'}^o\in (0,\infty)$ as $\b\to 0^+$ then the stochastic process 
$W^\beta_{kk'}(t)$ converges weakly to the solution of the stochastic differential equation
\beq\label{OUSDE2}
\ud W_{kk'}(t)
= \frac{2\left|\psi-3\right|}{d} W_{kk'}(t) \ud t +  \Lambda_{d,\psi}\ud B_{kk'}(t)
\eeq
with initial values $W_{kk'}(0) = W_{kk'}^o$. Here $B_{kk'}(t)$ is a standard one-dimensional Brownian motion. 
\end{theo}

We can solve Eq.~\eqref{OUSDE2} and obtain an unstable Ornstein-Uhlenbeck process as
\beq\label{Wkt}
W_{kk'}(t) = \left(W_{kk'}^o +  \Lambda_{d,\psi} \int_0^t \exp\left( -\frac{2\left| \psi-3 \right|}{d} s \right)  \ud B_{kk'}(s)\right) \exp\left(\frac{2\left| \psi-3 \right|}{d}  t \right).
\eeq
Let $C_{kk'}$ be defined as
\beq\label{Wktapprox}
C_{kk'} \equiv W_{kk'}^o + \Lambda_{d,\psi} \int_0^\infty \exp\left( -\frac{4\left| \psi-3 \right|}{d} s \right)   \ud B_{kk'}(s).
\eeq
We conclude that the following holds. 
\begin{enumerate}[label=(\roman*), topsep=0pt, leftmargin=*]
\item
$C_{kk'}$ is a normal variable with mean $W_{kk'}^o$ and variance $d\Lambda_{d,\psi}^2 / \left(4\left|\psi-3\right| \right)$;
\item
When $t$ is large $W_{kk'}(t)$ has the following approximation
\beq
W_{kk'}(t) \approx C_{kk'} \exp\left( \frac{2\left| \psi-3 \right|}{d} t\right).
\eeq
\end{enumerate}
To verify (i) above we have the It\^o integral in Eq.~\eqref{Wkt}
$$
\E\left(\Lambda_{d,\psi} \int_0^\infty \exp\left( -\frac{2\left| \psi-3 \right|}{d} s \right)  \ud B_{kk'}(s)\right)
= 0,
$$
and by using It\^o isometry 
\begin{align*}
&\E\left(\Lambda_{d,\psi} \int_0^\infty \exp\left( -\frac{2\left| \psi-3 \right|}{d} s \right)  \ud B_{kk'}(s)\right)^2
=                     \Lambda_{d,\psi}^2 \int_0^t \exp\left( -\frac{4\left| \psi-3 \right|}{d} s \right) \ud s  \\
&\hspace{1.5in}  \approx \Lambda_{d,\psi}^2 \int_0^\infty \exp\left( -\frac{4\left| \psi-3 \right|}{d} s \right) \ud s
 =                     \frac{d\Lambda_{d,\psi}^2}{4\left|\psi-3\right| }.
\end{align*}
The analysis above on the unstable Ornstein-Uhlenbeck process indicates that the process has the \emph{momentum} nature that when $t$ is large, it can be regarded as at a normally distributed location centered at 0 and grows exponentially. In Section \ref{sec:phase} we will see how the result in Theorem \ref{theo:SDE2} provides explanation on the escape from unstable equilibria.


\pb\section{Phase Analysis}\label{sec:phase}
In this section, we utilize the weak convergence results in Sections \ref{sec:infquan} and \ref{sec:SDE} to understand the dynamics of online ICA in different phases. For purposes of illustration and brevity, we restrict ourselves to the case of starting point $\vb^*$, a local maxima that has negative curvatures in every direction. In below we denote by $Z^\beta\asymp W^\beta$ as $\beta\to 0^+$ when the limit of ratio $Z^\beta / W^\beta \to 1$.

\vspace{-.05in}
\paragraph{Phase I (Escape from unstable equilibria).}
Assume we start from $\vb^*$, then $W_{kk'}^o = 0$ for all $k\ne k'$. We have from Eqs.~\eqref{Wkt} and \eqref{Wktapprox} that
\beq\label{Wtclosed}
\log\left( \frac{v_k^{(n)}}{v_{k'}^{(n)}} \right)^2 = \b^{1/2} W^\beta_{kk'}(n\beta)
\approx 
\left(\beta \frac{d\Lambda_{d,\psi}^2}{4\left|\psi-3\right| }\right)^{1/2}  \chi_{kk'}  \exp\left( \frac{2\left| \psi-3 \right|}{d}\cdot  \beta n  \right).
\eeq
Suppose $k_1$ is the index that maximizes $\left(  v_k^{(N_1^\beta)}\right)^2$ and $k_2$ maximizes $\left(  v_k^{(N_1^\beta)}\right)^2, k\ne k_1$. Then by Eq.~\eqref{Wtclosed} we know $\chi_{k_1k_2}$ is positive. By setting
$$
\log\left( v_{k_1}^{(N_1^\beta)} \right)^2 - \log\left( v_{k_2}^{(N_1^\beta)} \right)^2 = \log 2,
$$
we have from the construction in the proof of Theorem \ref{theo:SDE2} that as $\beta \to 0^+$
$$
N_1^\beta 
= \frac12 \left|\psi - 3 \right|^{-1} d \beta^{-1} \log \left( \left(\beta \frac{d\Lambda_{d,\psi}^2}{4\left|\psi-3\right| }\right)^{-1/2} \chi_{k_1k_2}^{-1}  \log 2 \right) \\
\asymp \frac14 \left|\psi - 3 \right|^{-1} d \beta^{-1} \log \left( \beta^{-1}  \right).
$$

\vspace{-.05in}
\paragraph{Phase II (Deterministic traverse).} 
By (strong) Markov property we can restart the counter of iteration, we have the max and second max
$$
\left( v_{k_1}^{(0)} \right)^2 = 2 \left( v_{k_2}^{(0)} \right)^2.
$$
Proposition \ref{prop:ODEarrival} implies that it takes time
$$
T \le |\psi-3|^{-1} \left( d-3 + 4\log (2\delta)^{-1} \right),
$$
for the ODE to traverse from $V_1^2 = 2/(d+1) = 2V_k^2$ for $k>1$. Converting to the timescale of the SGD, the second phase has the following relations as $\beta\to 0^+$
$$
N_2^\beta \asymp  T\beta^{-1} \le  |\psi-3|^{-1} \left( d-3 + 4\log (2\delta)^{-1} \right) \beta^{-1}.
$$

\vspace{-.05in}
\paragraph{Phase III (Convergence to stable equilibria).}
Again restart our counter. We have from the approximation in Theorem \ref{theo:SDE2} and Eq.~\eqref{Ukt} that
\begin{align*}
\E ( v_k^{(n)} )^2 
&=  (v_k^{(0)})^2 \exp\left( -2|\psi-3| \beta n \right) + \beta \psi_6 \int_0^{\beta n} \exp\left( -2|\psi-3| (t-s) \right)  \ud s \\
&= \frac{\beta\psi_6}{ 2|\psi-3| } + \left( (v_k^{(0)})^2 - \frac{\beta\psi_6}{ 2|\psi-3| }\right) \exp\left( -2\beta |\psi-3| n \right).
\end{align*}
In terms of the iterations $\vb^{(n)}$, note the relationship $\E \sin^2\angle(\vb, \eb_1)  = \sum_{k=2}^d v_k^2 =  1 - v_1^2.$ The end of ODE phase implies that $\E \sin^2\angle(\vb^{(0)}, \eb_1) = \delta$, and hence
$$
\E \sin^2\angle(\vb^{(n)}, \eb_1) 
= \frac{\beta(d-1)\psi_6}{ 2|\psi-3| } + \left( \delta  - \frac{\beta(d-1)\psi_6}{ 2|\psi-3| }\right) \exp\left( -2\beta |\psi-3| n \right).
$$
By setting
$$
\E \sin^2\angle(\vb^{(N_3^\beta)}, \eb_1) = (C_0+1) \cdot \frac{\beta(d-1)\psi_6}{2|\psi-3|},
$$
we conclude that as $\beta\to 0^+$
$$
N_3^\beta = \frac{1}{2\beta |\psi-3|} \log \left( \beta^{-1}\cdot  \frac{2|\psi-3| \delta  - \beta(d-1)\psi_6}{C_0(d-1)\psi_6}  \right)
\asymp \frac12 |\psi-3|^{-1} \beta^{-1} \log \left( \beta^{-1} \right).
$$

%

\section{Summary and discussions}\label{sec:dis}
In this paper, we take online ICA as a first step towards understanding the global dynamics of stochastic gradient descent. For general nonconvex optimization problems such as training deep networks, phase-retrieval, dictionary learning and PCA, we expect similar multiple-phase phenomenon. It is believed that the flavor of asymptotic analysis above can help identify a class of stochastic algorithms for nonconvex optimization with statistical structure.

Our continuous-time analysis also reflects the dynamics of the algorithm in discrete time. This is substantiated by Theorems \ref{theo:ODE}, \ref{theo:SDE1} and \ref{theo:SDE2} which rigorously characterize the convergence of iterates to ODE or SDE by shifting to different temporal and spatial scales. In detail, our results imply when $\beta\to 0^+$:

\vspace{-.05in}
\begin{itemize}
\item[]
Phase I takes iteration number $N^\beta_1 \asymp (1/4)|\psi-3|^{-1}d\cdot  \beta^{-1}\log(\beta^{-1})$; 
\item[]
Phase II takes iteration number $N^\beta_2 \asymp |\psi-3|^{-1}d\cdot  \beta^{-1}$;
\item[]
Phase III takes iteration number $N^\beta_3 \asymp (1/2)|\psi-3|^{-1}\cdot  \beta^{-1} \log(\beta^{-1})$.
\end{itemize}
\vspace{-.05in}

After the three phases, the iteration reaches a point that is $C \cdot \left(\psi_6 |\psi-3|^{-1}\cdot d\beta\right)^{1/2}$ distant on average to one local minimizer. As $\beta\to 0^+$ we have $N^\beta_2 / N^\beta_1 \to 0$. This implies that the algorithm demonstrates the \emph{cutoff} phenomenon which frequently occur in discrete-time Markov processes \cite[Chap.~18]{MarkovMixing}. In words, the Phase II where the objective value in Eq.~\eqref{opt1} drops from $1-\ep$ to $\ep$ is a short-time phase compared to Phases I and III, so the convergence curve illustrated in the right figure in Figure \ref{fig:illu} instead of an exponentially decaying curve. As $\beta\to 0^+$ we have $N^\beta_1 / N^\beta_3 \asymp d/2$, which suggests that Phase I of escaping from unstable equlibria dominates Phase III by a factor of $d/2$.

%
%
%
%
%
%
%
%

\begingroup
{\small
\bibliographystyle{ims}
\setlength{\bibsep}{0pt}
\bibliography{OnlineTensor}
}
\endgroup

\pagebreak
\appendix
\section{Detailed Proofs in Sections \ref{sec:setting} and \ref{sec:infquan}}
\subsection{Proof of Lemma \ref{lemm:vTY}}\label{ssec:proof,lemm:vTY}
\begin{proof}
We only need to show
\beq\label{EvTY}
\E\left( \vb^\top \bY \right)^4
= 3 + (\psi-3) \sum_{i=1}^d v_i^4.
\eeq
Note due to the following well-known expansion \cite{bronshtein2013handbook}
$$
\left(\sum  x_i\right)^4 
= \sum  x_i^4 + 4\sum x_{i}^3 x_{j}  + 6\sum x_{i}^2 x_{j}^2+ 12\sum x_{i_1}^2 x_{i_2} x_{i_3} + 24\sum x_{i_1} x_{i_2} x_{i_3} x_{i_4}.
$$
where the summations above iterate through all monomial terms. Plugging in $x_i = v_i Y_i$ and taking expectations, we conclude that under Assumption \ref{assu:distribution}
\beq\label{der1}\begin{split}
\E\left( \vb^\top \bY \right)^4
&= \sum_{i=1}^d v_i^4 \E\left( Y_i^4 \right)
+6\sum_{1\le i< j\le d} v_i^2 v_j^2 \E \left(Y_i^2\right) \E\left( Y_j^2\right) \\
&= \psi\sum_{i=1}^d v_i^4 +6\sum_{1\le i< j\le d} v_i^2 v_j^2.
\end{split}\eeq
Note that from the constraint of our optimization problem Eq.~\eqref{opt1}, we have
\beq\label{der2}
1 = \|\vb\|^4 = \left(\sum_{i=1}^d v_i^2\right)^2
= \sum_{i=1}^d v_i^4 + 2\sum_{1\le i< j\le d} v_i^2 v_j^2.
\eeq
Combining both Eqs.~\eqref{der1} and \eqref{der2} we conclude Eq.~\eqref{EvTY} and hence the lemma.
\end{proof}

\pb\subsection{Proof of Proposition \ref{prop:Markov}}\label{ssec:proof,prop:Markov}
\begin{proof}
Let $\F_n = \sigma(\ub^{(n')}: n'\le n)$ be the $\sigma$-field filtration generated by the iteration $\ub^{(n)}$, viewed as a stochastic process. From the recursion equation in Eq.~\eqref{tensoralgo} we have a Markov transition kernel $p(\ub, \cS)$ such that for each Borel set $\cA\subseteq \cS^{d-1}$
$$
\P\left( \ub^{(n)} \in \cA \mid \F_{n-1} \right) = p(\ub^{(n-1)}, \cA).
$$
Therefore it is a time-homogeneous Markov chain. The strong Markov property holds directly from Markov property, see \cite{DurrettPTE} as a reference. This proves Proposition \ref{prop:Markov}.

\end{proof}

\pb\subsection{Proof of Theorem \ref{theo:ODE}}\label{ssec:proof,theo:ODE}

We first use the standard one-step analysis and conclude the following proposition, whose proof is deferred to Subsection \ref{ssec:infquan}.

\begin{prop}\label{prop:infquan}
For brevity let $\vb= \vb^{(0)}$ and $\bY = \bY^{(1)}$, separately. Under Assumption \ref{assu:distribution}, when
\beq\label{betacond}
B^2 \beta\le 2/3,
\eeq
for each $k= 1,2,\ldots,d$ and $n\ge 0$ we have the following:
\begin{enumerate}[label=(\roman*), topsep=0pt, leftmargin=*]
\item
There exists a random variable $R_k$ that depends solely on $\vb, \bY$ with $|R_k| \le 9B^4\b^2$ almost surely, such that the increment $v_k^{(1)} - v_k^{(0)}$ can be represented as
\beq\label{incremental}
v_k^{(1)} - v_k^{(0)}
=  \b\left(   \left(\vb^\top \bY\right)^3 Y_k - v_k \left(\vb^\top \bY\right)^4 \right) + R_k;
\eeq 

\item
The increment of $v_k$ on coordinate $k$ has the following bound
\beq\label{infbdd}
\left| v_k^{(1)} - v_k^{(0)} \right| \le 8 B^2\b;
\eeq

\item
There exists a deterministic function $E_k(\vb)$ with $\sup_{\vb\in\cS^{d-1}} \left| E_k(\vb) \right|  \le 9B^4\b^2$, such that the conditional expectation of the increment $v_k^{(1)} - v_k^{(0)}$ is
\beq \label{infmean} 
\E\left[   v_k^{(1)} - v_k^{(0)}       \, \big|\, \vb^{(0)} =\vb \right]  
= \beta \left|\psi-3\right| v_k \left( v_k^2 -  \sum_{i=1}^d v_i^4  \right)
+ E_k(\vb).
\eeq
\end{enumerate}
\end{prop}

In Proposition \ref{prop:infquan}, (i) characterizes the relationship between the increment on $v_k$ and the online sample, and (ii) bounds such increment. From (iii) we can compute the infinitesimal mean and variance for SGD for tensor method and conclude that as the stepsize $\beta\to 0^+$, the iterates generated by Eq.~\eqref{tensoralgo}, under the time scaling that speeds up the algorithm by a factor $\beta^{-1}$, can be globally approximated by the solution to the following ODE system in Eq.~\eqref{ODE} as
$$
\frac{\ud V_k}{\ud t} = \left|\psi-3\right| V_k\left( V_k^2 - \sum_{i=1}^d V_i^4 \right), \qquad k=1,\ldots,d.
$$
To characterize such approximation we use theory of weak convergence to diffusions \cite{stroock1979multidimensional,EthierKurtz}. We remind the readers of the definition of weak convergence $Z^\beta \Rightarrow Z$ in stochastic processes: for any $0\le t_1<t_2<\cdots < t_n$ the following convergence in distribution occurs as $\b\to 0^+$
$$
\left( Z^\b(t_1), Z^\b(t_2), \ldots,Z^\b(t_n)\right)
\dcv \left(Z(t_1),Z(t_2), \ldots, Z(t_n)\right).
$$
To highlight the dependence on $\beta$ we add it in the superscipts of iterates $\vb^{\beta,(n)} = \vb^{(n)}$.

\begin{proof}[Proof of Theorem \ref{theo:ODE}]
Let $V_k^\beta(t) = v_k^{\beta,(\lfloor t\beta^{-1} \rfloor)}$. Proposition \ref{prop:infquan} implies for coordinate $k$ $V_k^\b(t)$ satisfies
$$
V_k^\b(\b) - V_k^\b(0) =  \b\left(   (\vb^\top \bY)^3 Y_k - v_k (\vb^\top \bY)^4 \right) + R_k,
$$
where $|R_k|\le 9B^4\b^2$. Eq.~\eqref{infmean} implies that if the infinitesimal mean is \cite{EthierKurtz} 
\begin{align*}
\frac{\ud}{\ud t} \E V_k^\b(t) \bigg|_{t=0}
&=\b^{-1} \E\left[  V_k^\b(\b) - v_k    \, \big|\, \bV^\b(0) =\vb \right]   \\
&= \left| \psi-3 \right| v_k \left( v_k^2 -  \sum_{i=1}^d v_i^4 \right) + \O(B^4\b).
\end{align*}
Using Eq.~\eqref{infbdd} we have the infinitesimal variance
\begin{align*}
\frac{\ud}{\ud t} \E( V_k^\b(t) - v_k )^2 \bigg|_{t=0}
&=\b^{-1} \E\left[  (V_k^\b(\b) - v_k )^2   \, \big|\, V_k^\b(0) =\vb \right]  \\
&\le \b^{-1} \cdot C^2 B^4\b^2,
\end{align*}
which tends to 0 as $\beta\to 0^+$. Let $V_k(t)$ be the solution to ODE system Eq.~\eqref{ODE} with initial values $V_k(0) = v_k^{\beta,(0)}$. Applying standard infinitesimal generator argument \cite[Corollary 4.2 in Sec.~7.4]{EthierKurtz} one can conclude that as $\b\to 0^+$, the Markov process $V_k^\b(t)$ converges weakly to $V_k(t)$. 

\end{proof}

\pb\subsection{Proof of Proposition \ref{prop:ODEarrival}}\label{ssec:proof,prop:ODEarrival}
For simplicity we denote in the proofs that the initial value $V_k(0) = V_k, k=1,\ldots, d$. Also, throughout this subsection we assume without loss of generality that $V_1^2$ is maximal among $V_k^2$, $k=1,\ldots,d$, and furthermore
\beq\label{v1max}
V_1^2 \ge  2 \max_{k>1}  V_k^2.
\eeq

\begin{lemm}\label{lemm:gappersist}
For $\bV \in \cS^{d-1}$ that satisfies Eq.~\eqref{v1max}, then we have for all $t\ge 0$
\beq\label{gappersist_ODE}
\left(V_1(t)\right)^2  \ge 2 \max_{k>1} \left(V_k(t)\right)^2.
\eeq
\end{lemm}

\begin{proof}
We compare the coordinate between two distinct coordinates $i,j$ and have by calculus that for all $k>1$
\beq\label{logincrease}
\frac{\ud}{\ud t} \log \left(\frac{V_k(t)}{V_1(t)}\right)^2 = 2\left|\psi-3\right| \left( V_k^2(t) - V_1^2(t) \right).
\eeq
So if initially Eq.~\eqref{v1max} is valid then $\log \left(V_k^2(t) / V_1^2(t)\right)$ is nondecreasing, which indicates for all $t > 0$
$$
 \log \left(\frac{V_k(t)}{V_1(t)}\right)^2  \le  \log \left(\frac{V_k}{V_1}\right)^2 \le \log \frac12.
$$
Rearranging the above display and taking maximum over $k=2,\ldots, d$ gives Eq.~\eqref{gappersist_ODE}.

\end{proof}

We then establish a lemma that gives the lower bound of drift term related to $V_1$. To bound the bracket term on the right hand of ODE, one has
\beq\label{v1drift_upperbound}
V_1^2 - \sum_{k=1}^d V_k^4
= V_1^2 - V_1^4 - \sum_{k=2}^d V_k^4
\le  V_1^2 (1 - V_1^2).
\eeq
which gives us an upper bound. To obtain a lower bound estimate we first state a lemma stating that the gap between the first and all other coordinates is nondecreasing.

\begin{lemm}\label{lemm:v1drift_lowerbound}
For $\bV \in \cS^{d-1}$ that satisfies Eq.~\eqref{v1max} we have
\beq\label{v1drift_lowerbound}
V_1^2 (1-V_1^2)
\ge  V_1^2 - \sum_{k=1}^d V_k^4
\ge  \frac{V_1^2}{2} (1-V_1^2).
\eeq
\end{lemm}

\begin{proof}
Note H\"older's inequality gives
\beq\label{holder}
\sum_{k > 1} V_k^4 \le \left(\max_{k > 1} V_k^2\right) \left( \sum_{k > 1} V_k^2 \right),
\eeq
where the equality in the above display holds when $V_2^2 = \cdots = V_d^2$. Using Eq.~\eqref{v1max} and \eqref{holder} one has
\beq\label{ineq1}
\begin{split}
V_1^2 - \sum_{k=1}^d V_k^4 
&\ge V_1^2 - V_1^4 - \left(\max_{k > 1} V_k^2\right) \left( 1-V_1^2 \right) \\
&\ge V_1^2 - V_1^4 - \frac{V_1^2}{2} \left( 1-V_1^2 \right) 
= \frac{V_1^2}{2} \left( 1-V_1^2 \right).
\end{split}
\eeq
This completes the proof.

\end{proof}

\begin{lemm}\label{lemm:ODEaux}
For the ODE in Eq.~\eqref{ODEaux} which is
\beq\label{ODEaux2}
\frac{\ud y}{\ud t} = y^2 \left( 1-y \right),
\eeq
with $y(0) = 2/(d+1)$. By letting $T_0$ be such that $y(T_0) = 1-\delta$, we have
\beq\label{Tupper2}
T_0 \le  d-3 + 4\log (2\delta)^{-1}.
\eeq
\end{lemm}

\begin{proof}
Let $T_1$ be the traverse time from $2/(d+1)$ to $1/2$, and $T_2$ be from $1/2$ to $1-\delta$. We have for $y\in [0,1/2]$
$$
\frac12 y^2 \le \frac{\ud y}{\ud t} \le y^2.
$$
Therefore by comparison theorem of ODE \cite{hirsch2012differential}, $T_1^* \le T_1 \le 2T_1^*$
where $y_1(t) = \frac{y_0}{1-y_0 t}$ solves $\ud y_1 / \ud t = y_1^2$, $y_1(0) = 2/(d+1)$. Letting $y_1(T_1^*) = 1/2$ we obtain $T_1^* = (d-3)/2$. For $T_2$ we note for $y\in [1/2,1]$
$$
\frac14(1-y) \le \frac{\ud y}{\ud t} \le 1-y.
$$
Comparing with $y_2(t) = 1 - (1/2)e^{-t}$ which solves the ODE $\ud y_2 / \ud t = 1-y_2$ with $y_2(0) = 1/2$, we have $T_2^* = \log (2\delta)^{-1}$ such that $y_2(T_2^*) = 1-\delta$. To summarize we have
$$
T_0 \le 2T_1^* + 4 T_2^* = d-3 + 4\log (2\delta)^{-1}.
$$

\end{proof}

\begin{proof}[Proof of Proposition \ref{prop:ODEarrival}]
From the ODE in Eq.~\eqref{ODE} we have
$$
\frac{\ud V_1^2}{\ud t} = 2\left|\psi-3\right| V_1^2 \left( V_1^2 - \sum_{i=1}^d V_i^4 \right).
$$
Combining both Lemmas \ref{lemm:gappersist} and \ref{lemm:v1drift_lowerbound} we have
$$
2\left|\psi-3\right| V_1^4 \left( 1 - V_1^2 \right)
\ge \frac{\ud V_1^2}{\ud t} 
\ge \left|\psi-3\right| V_1^4 \left( 1 - V_1^2 \right).
$$
If the starting value of algorithm has $V_1^2 \ge 2\max_{k>1} V_k^2$ then $V_1^2 \ge 2/(d+1)$. By comparison theorem in ODE \cite{hirsch2012differential} we know $V_1^2(t)$ runs the auxiliary ODE Eq.~\eqref{ODEaux} at a nonconstant rate within $\left[\left|\psi-3\right|, 2\left|\psi-3\right|\right]$. Therefore the time
$$
\frac12 |\psi-3|^{-1} T_0 \le T \le |\psi-3|^{-1} T_0.
$$
Combining with Lemma \ref{lemm:ODEaux} we are done.

\end{proof}

\pb\section{Detailed Proofs in Section \ref{sec:SDE}}
\subsection{Proof of Theorem \ref{theo:SDE1}}\label{ssec:proof,theo:SDE1}
\begin{proof}
Proposition \ref{prop:infquan} implies for $U_k^\b(t) = \b^{-1/2} v_k^{\beta,(\lfloor t\beta^{-1} \rfloor)}$, under the conditions in Theorem \ref{theo:SDE1} the one-step increment on coordinate $k$ is
$$
U_k^\b(\b) - U_k^\b(0) 
=  \beta^{-1/2} \left(v_k^{\beta,(1)} - v_k^{\beta,(0)}\right)
=  \beta^{-1/2}\beta \left(   \left(\vb^\top \bY\right)^3 Y_k - v_k \left(\vb^\top \bY\right)^4 \right) + \beta^{-1/2}R_k.
$$
Eq.~\eqref{infmean} implies that the infinitesimal mean is
\begin{align*}
\frac{\ud}{\ud t} \E U_k^\b(t) \bigg|_{t=0}
&=\beta^{-1} \E\left[  U_k^\b(\b) - U_k^\b(0)    \, \big|\, \bV^\b(0) = \vb, \bU^\b(0) = \ub \right]   \\
&= \beta^{-1}\beta^{-1/2}\cdot \beta \left|\psi-3\right| v_k \left( v_k^2 -  \sum_{i=1}^d v_i^4  \right) + \beta^{-1}\beta^{-1/2}\cdot E_k(\vb) \\
&= - \left|\psi-3\right|  u_k  + o(1).
\end{align*}
Using Eq.~\eqref{infbdd} we have the infinitesimal variance
\begin{align*}
\frac{\ud}{\ud t} \E( U_k^\b(t) - U_k^\b(0) )^2 \bigg|_{t=0}
&=\b^{-1} \E\left[  (U_k^\b(\b) - U_k^\b(0) )^2   \, \big|\, \bV^\b(0) =\vb \right]  \\
&= \beta^{-2} \E \left[ \left(v_k^{\beta,(1)} - v_k^{\beta,(0)} \right)^2 \,\big|\, \bV^\b(0) =\vb \right] \\ 
&= \E(Y_1^3Y_k)^2 + o(1) = \psi_6 + o(1).
\end{align*}
In addition $\left|U_k^\b(t) - U_k^\b(0)\right| \le C B^2\b$. Applying standard infinitesimal generator argument \cite[Sec.~7.4]{EthierKurtz} one can conclude that as $\b\to 0^+$, the Markov process $U_k^\b(t)$ {\it converges weakly} to $U_k(t)$ the solution to Eq.~\eqref{OUSDE1}. 

\end{proof}

\pb\subsection{Proof of Theorem \ref{theo:SDE2}}\label{ssec:proof,theo:SDE2}
We first prove an auxillary lemma on moment calculations. Proof is deferred to Subsection \ref{ssec:proof,lemm:aux_moment}
\begin{lemma}\label{lemm:aux_moment}
We have for each $k=1,\ldots, d$ the following moment expressions:
$$
\EE\left( \sum_{i=1}^d Y_i \right)^6 Y_k^2
=
\psi_8 +  16(d-1) \psi_6  +  15(d-1) \psi_4^2 + 60(d-1)(d-2) \psi_4
+ 30(d-1)(d-2)(d-3),
$$
and
$$
\EE\left( \sum_{i=1}^d Y_i \right)^8
=
 d\psi_8 + 28d(d-1)\psi_6 
+ 35d(d-1)(1+ 12 (d-1) (d-2)) \psi_4
 + 105 d(d-1) (d-2)(d-3).
$$
\end{lemma}

\begin{proof}[Proof of Theorem \ref{theo:SDE2}]
Note from the definition in Eq.~\eqref{Wk} we have for distinct coordinate pair $k, k'$,
\beq\label{Wkk'}
\beta^{1/2} W_{kk'} =  \log\left( v_k^2 \right) - \log\left( v_{k'}^2 \right).
\eeq
By symmetry we without loss of generality that $v_k^{(0)}, v_{k'}^{(0)} > 0$ and hence
$$
W^\beta_{kk'}(\beta) - W^\beta_{kk'}(0)
= 2 \beta^{-1/2} \log \left(\frac{v_k^{(1)}}{v_k^{(0)}}\right)
- 2 \beta^{-1/2} \log \left(\frac{v_{k'}^{(1)}}{v_{k'}^{(0)}}\right).
$$
However Proposition \ref{prop:infquan} indicates that
$$
\log \left(\frac{v_k^{(1)}}{v_k^{(0)}}\right) =
\frac{v_k^{(1)}-v_k^{(0)}}{v_k^{(0)}} + \O(\beta^2)
= \beta  \left( \vb^\top \bY  \right) ^3 \frac{Y_k}{v_k} -
\beta (\vb^\top \bY)^4 + \O(\beta^2),
$$
and analogously for $k'$. For infinitesimal mean
\begin{align*}
\frac{\ud}{\ud t} \E( W_{kk'}^\b(t) - W_{kk'}^\b(0) ) \bigg|_{t=0}
&=\beta^{-1}\E\left[ W^\beta_{kk'}(\beta)- W^\beta_{kk'}(0) \,\big|\, W^\beta_{kk'}(0) = W_{kk'} \right]  \\
&= \beta^{-1}\cdot 2\beta^{-1/2} \E\left[ \beta  \left( \vb^\top \bY  \right) ^3 \left(\frac{Y_k}{v_k} - \frac{Y_{k'}}{v_{k'}}\right) + \O(\beta^2) \right]
\end{align*}
Since $\E\left[ \left( \vb^\top \bY  \right) ^3 (Y_k/v_k)  \right] = 3 + (\psi-3) v_k^2$, and analogously for $k'$, and also $(v_k^{\beta,(0)})^2 \to 1$, we conclude from Eq.~\eqref{Wkk'} that
\begin{align*}
2(\psi-3)\beta^{-1/2} \left( (v_k^{\beta,(0)})^2 - (v_{k'}^{\beta,(0)})^2 \right) 
&=
2(\psi-3)\beta^{-1/2}\cdot \frac{1}{d} \cdot \log\left( \frac{v_k^{\beta,(0)} }{ v_{k'}^{\beta,(0)} } \right)^2  + \O(\beta)
\\&\to 
\frac{2(\psi-3)}{d} W_{kk'}.
\end{align*}
For infinitesimal variance
\begin{align*}
\frac{\ud}{\ud t} \E( W_{kk'}^\b(t) - W_{kk'}^\b(0) )^2 \bigg|_{t=0}
&=\beta^{-1} \E\left[ \left(W^\beta_{kk'}(\beta)- W^\beta_{kk'}(0)\right)^2 \,\big|\, W^\beta_{kk'}(0) = W_{kk'} \right] \\ 
&= 4 \beta^{-2}\E\left[\left( \log \left(\frac{v_k^{(1)}}{v_k^{(0)}}\right) - \log \left(\frac{v_{k'}^{(1)}}{v_{k'}^{(0)}}\right)\right)^2 \,\bigg|\, W^\beta_{kk'}(0) = W_{kk'} \right]
\\&= 
4 
\EE\left[
 \left( \vb^*\,^\top \bY  \right) ^3 \frac{Y_k}{v_k}
-
 \left( \vb^*\,^\top \bY  \right) ^3 \frac{Y_{k'}}{v_{k'}}
\right]^2 
+
\O(\beta).
\end{align*}
Note the second-order term
$$ 
\EE\left[
\left( \vb^\top \bY  \right) ^6
 \left(  \frac{Y_k}{v_k}
\right)^2 
\,\bigg|\, W^\beta_{kk'}(0) = W_{kk'} \right]
=
 4 d^{-2}\cdot \E\left( \sum_{i=1}^d Y_i\right)^6 Y_k^2
 \equiv  4 d^{-2}Q_1,
$$
and similarly for index $k'$. For the cross term in the expectation we have
\begin{align*}
\EE\left[
\left( \vb^\top \bY  \right) ^6
 \frac{Y_k}{v_k}
\cdot
 \frac{Y_{k'}}{v_{k'}}
\,\bigg|\, W^\beta_{kk'}(0) = W_{kk'} \right]
=
 4 d^{-2}\cdot \E\left( \sum_{i=1}^d Y_i\right)^6 Y_k Y_{k'}
 \equiv  4 d^{-2}Q_2.
\end{align*}
From standard polynomial manipulations we have
$$
dQ_1 + d(d-1)Q_2 = 
 d\psi_8 + 28d(d-1)\psi_6 
+ 35d(d-1)(1+ 12 (d-1) (d-2)) \psi_4
 + 105 d(d-1) (d-2)(d-3),
$$
and
$$
Q_1 = \psi_8 +  16(d-1) \psi_6  +  15(d-1) \psi_4^2 + 60(d-1)(d-2) \psi_4
+ 30 (d-1)(d-2)(d-3).
$$
Therefore
\begin{align*}
Q_1 - Q_2
&=
\frac{d^2 Q_1 - d Q_1 - d(d-1)Q_2}{d(d-1)}
\\&= 
\psi_8 +  (16d-28) \psi_6  +  15d \psi_4^2 
-5 (72 d^2-228 d+175)\psi_4
+ 15(2d-7)(d-2)(d-3).
\end{align*}
Summarize the above calculations we obtain as $\beta\to 0^+$
\begin{align*}
\frac{\ud}{\ud t} \E( W_{kk'}^\b(t) - W_{kk'}^\b(0) )^2 \bigg|_{t=0}
&= 
4 
\EE\left[
 \left( \vb^*\,^\top \bY  \right) ^3 \frac{Y_k}{v_k}
-
 \left( \vb^*\,^\top \bY  \right) ^3 \frac{Y_{k'}}{v_{k'}}
\right]^2 
+
\O(\beta)
\\&= 
8d^{-2}(Q_1-Q_2)
+
\O(\beta).
\end{align*}
Combining the last two displays concludes the theorem.

\end{proof}

\section{Proof of Auxillary Results}
\subsection{Proof of Proposition \ref{prop:infquan}}\label{ssec:infquan}

For $\vb^{(0)} = \vb\in \cS^{d-1}$ the update equation becomes
$$
\vb^{(1)} = \big\| \vb + \beta \left(\vb^\top \bY \right)^3 \bY \big\|^{-1}     \left(\vb + \beta \left(\vb^\top \bY \right)^3 \bY \right).
$$
For the simplicity for discussion we prove under the condition $\psi > 3$ (the case of $\psi < 3$ is analogous). To prove Proposition \ref{prop:infquan} in the case of $\psi > 3$, we first introduce

\begin{lemm}\label{lemm:taylor}
For $x\in [0,1)$ we have
\beq\label{taylor_err}
\left| (1+x)^{-1/2} - 1 + \frac{x}{2} \right| \le 2 \left(\frac{x}{2}\right)^2.
\eeq
\end{lemm}
\begin{proof}
Taylor expansion suggests for $|x| < 1$
\begin{align*}
\left( 1+x\right)^{-1/2} = \sum_{n=0}^\infty \binom{-\frac12}{n} x^n = 1-\frac12 x + \frac38 x^2 - \frac5{16} x^3 + \cdots
\end{align*}
which is an alternating series for $x\in [0,1)$, whereas the absolute terms approach to 0 monotonically
$$
\left|\binom{-\frac12}{n+1} x^{n+1}\right| \le \left|\binom{-\frac12}{n} x^n\right|.
$$
This indicates that for $x\in [0,1)$
$$
\left| (1+x)^{-1/2} - 1 + \frac{1}{2}x \right| \le \frac{3}{8} x^2 \le \frac{1}{2} x^2,
$$
which completes the proof of Lemma \ref{lemm:taylor}.

\end{proof}

\begin{proof}[Proof of Proposition \ref{prop:infquan}]
When Eq.~\eqref{betacond} is satisfied, and noting $|\vb^\top \bY|^2 \le \|\bY\|^2\le B$, we have from Eq.~\eqref{betacond}
$$
\b (\vb^\top \bY)^4 + \frac12\b^2 (\vb^\top \bY)^6 \|\bY\|^2 \le B^2 \b + \frac12 B^4 \b^2
\le  \frac43 B^2 \b  < 1,
$$
and hence from Eq.~\eqref{taylor_err} in Lemma \ref{lemm:taylor} there exists a $Q_1(\vb, \bY)$ with
$$
\left|Q_1(\vb, \bY)\right| \le 2 \left( \b (\vb^\top \bY)^4 + \frac12\b^2 (\vb^\top \bY)^6 \|\bY\|^2\right)^2
\le \frac{32}{9} B^4 \b^2,
$$
such that, with $Q_2(\vb, \bY) = - \frac12\b^2 (\vb^\top \bY)^6 \|\bY\|^2 + Q_1(\vb, \bY)$, we have
\begin{align}\label{Qder1}
\|\vb+\b (\vb^\top \bY)^3 \bY \|^{-1} 
&= \left(1 + 2\b (\vb^\top \bY)^4 + \b^2 (\vb^\top \bY)^6 \|\bY\|^2 \right)^{-1/2} \nonumber \\
&= 1 - \b (\vb^\top \bY)^4 - \frac12\b^2 (\vb^\top \bY)^6 \|\bY\|^2 + Q_1(\vb, \bY) \nonumber \\
&= 1 - \b (\vb^\top \bY)^4 + Q_2(\vb, \bY),
\end{align}
where
\beq\label{Qder2}
\left| Q_2(\vb, \bY)\right|
\le \frac12 B^4 \b^2 + \frac{32}{9} B^4 \b^2 = \frac{73}{18} B^4 \b^2.
\eeq
Using Eqs.~\eqref{Qder1} and \eqref{Qder2} we have
\begin{align}\label{recvk}
\hat v_k - v_k
&= \left\|\vb+\b (\vb^\top \bY)^3 \bY \right\|^{-1} \left(v_k + \b \left(\vb^\top \bY\right)^3 Y_k\right) - v_k \nonumber \\
&=\left( 1 - \b (\vb^\top \bY)^4 + Q_2(\vb, \bY) \right)  \left(v_k + \b \left(\vb^\top \bY\right)^3 Y_k\right) - v_k \nonumber \\
&=  \b\left(   (\vb^\top \bY)^3 Y_k - v_k (\vb^\top \bY)^4 \right) + Q_3(\vb, \bY),
\end{align}
where
\beq\label{Qk}
Q_3(\vb, \bY)  =  \left(v_k + \b \left(\vb^\top \bY\right)^3 Y_k\right)  Q_2(\vb, \bY)
  - \b^2 (\vb^\top \bY)^7 Y_k
\eeq
which has the following estimate
\beq\label{Qkest}
\begin{split}
\left| Q_3(\vb, \bY) \right|
&\le \left|v_k + \b \big(\vb^\top \bY\big)^3 Y_k\right| \left| Q_2(\vb, \bY) \right| + \b^2\left| \big(\vb^\top \bY\big)^7 Y_k\right| \\
&\le \left( 1+ B^2\b \right) \frac{73}{18} B^4 \b^2 + B^4\b^2
\le 9B^4\b^2.
\end{split}\eeq
Denoting $Q_3(\vb, \bY)$ by the random variable $R_k$, Eqs.~\eqref{recvk}, \eqref{Qk}, \eqref{Qkest} together concludes (i) of Prop.~\ref{prop:infquan}. 

For (ii), note Eq.~\eqref{incremental} gives
$$
\left|  v_k^{(1)} - v_k^{(n)}  \right| \le
\beta \left( \|\bY\|^2 + \|\bY\|^2 \right) + 9B^4\b^2 \le 8B^2\b,
$$
so it is concluded. 

For (iii), we set $E_k(\vb) = \E\left[  R_k \mid \vb^{(0)} =\vb\right]$. Under Assumption \ref{assu:distribution} we take conditional expectation on $\vb^{(n)} = \vb$ on both sides of Eq.~\eqref{incremental} to obtain
\beq\label{Edv1}
\begin{split}
\E\left[   v_k^{(1)} - v_k^{(0)}       \, \big|\, \vb^{(0)} =\vb \right]  
&= \b\E\left[   \left(\vb\,^\top \bY\right)^3 Y_k - v_k \left(\vb^\top \bY \right)^4   \, \big|\, \vb^{(0)} =\vb \right] + \E\left[  R_k \mid \vb^{(0)} =\vb\right] \\
&= \b  (\psi-3) v_k \left( v_k^2 -  \sum_{i=1}^d v_i^4  \right)
+ E_k(\vb).
\end{split}
\eeq
Similar to the proof of Lemma \ref{lemm:vTY} in Subsection \ref{ssec:proof,lemm:vTY} we quote another polynomial expansion \cite{bronshtein2013handbook}
$$
\left(\sum  x_i\right)^3
= \sum  x_i^3 + 3\sum x_{i}^2 x_{j}  + 6\sum x_{i_1} x_{i_2} x_{i_3}.
$$
where the summations above iterate through all monomial terms. Plugging in $x_i = v_i Y_i$ and taking conditional expectations, we conclude that under Assumption \ref{assu:distribution}
\beq\label{der1s}\begin{split}
\E\left[\left( \vb^\top \bY \right)^3 Y_k   \, \big|\, \vb^{(0)} =\vb \right]  
&= v_k^3 \E\left( Y_i^4 \right)
+3\sum_{i: i\ne k} v_i^2 v_k \E \left(Y_i^2\right) \E\left( Y_k^2\right) \\
&= \psi v_k^3 + 3(1-v_k^2) v_k = 3v_k + (\psi-3) v_k^3.
\end{split}\eeq
In Eq.~\eqref{EvTY} we have
\beq\label{Edv2}
\begin{split}
\E\left[ \left( \vb^\top \bY\right)^3 Y_k - v_k \left(\vb^\top \bY \right)^4    \, \big|\, \vb^{(0)} =\vb \right]  
&= 3v_k  +(\psi-3) v_k^3 - v_k\left(3 +(\psi-3) \sum_{i=1}^d v_i^4\right) \\
&= (\psi-3) v_k \left( v_k^2 - \sum_{i=1}^d v_i^4\right).
\end{split}
\eeq
Combining Eqs.~\eqref{Edv1} and \eqref{Edv2} completes the proof.

\end{proof}

\subsection{Proof of Lemma \ref{lemm:aux_moment}}\label{ssec:proof,lemm:aux_moment}
\begin{proof}
As in proof of Lemmas \ref{lemm:vTY} and Proposition \ref{prop:infquan}, we have the final polynomial expansions \cite{bronshtein2013handbook} that
$$
\left(\sum  x_i\right)^6
= \sum  x_i^6 + 15\sum x_i^4 x_j^2 
+ 90\sum x_i^2 x_j^2 x_k^2
+ \mbox{terms that has odd-order factors},
$$
and using some combinatorics counting we have
\begin{align*}
\left(\sum  x_i\right)^8
= \sum  x_i^8 + 28\sum x_i^6 x_j^2 
&+ 70\sum x_i^4 x_j^4 
+ 420\sum x_i^4 x_j^2 x_k^2 
\\&+ 2520\sum x_i ^2x_j^2 x_k^2 x_l^2
+ \mbox{terms that has odd-order factors}.
\end{align*}
Therefore to show the first equality, note from Assumption \ref{assu:distribution} we can assume WLOG that $k=1$. Thus
\begin{align*}
&\E\left( \sum_{i=1}^d Y_i \right)^6 Y_1^2
= \sum_{i=1}^d \E Y_i^6 Y_1^2 
+ 15 \sum_{1\le i< j\le d} \E Y_i^4 Y_j^2 Y_1^2 
+ 15 \sum_{1\le i<j\le d} \E Y_j^4 Y_i^2 Y_1^2 
\hspace{1in}
\\&
\hspace{1.5in}
+ 90 \sum_{1\le i<j<k\le d} \E Y_i^2 Y_j^2 Y_k^2 Y_1^2 
\end{align*}
\begin{align*}
\\&\quad= 
\E Y_1^8 
+ 15 \sum_{2\le j\le d} \E Y_1^6 Y_j^2 
+ 15 \sum_{2\le j\le d} \E Y_1^4 Y_j^4 
+ 90 \sum_{2\le j<k\le d} \E Y_1^4 Y_j^2 Y_k^2 
+ \sum_{i=2}^d \E Y_i^6 Y_1^2 
\\&\hspace{1in}
+ 15 \sum_{2\le i< j\le d} \E Y_i^4 Y_j^2 Y_1^2 
+ 15 \sum_{2\le i<j\le d} \E Y_j^4 Y_i^2 Y_1^2 
+ 90 \sum_{2\le i<j<k\le d} \E Y_i^2 Y_j^2 Y_k^2 Y_1^2 
\\&\quad = 
\psi_8 + 15(d-1)\psi_6 + 15(d-1)\psi_4^2 + 90\binom{d-1}{2}\psi_4
+(d-1) \psi_6 
\\&\hspace{1in}
+ 15\binom{d-1}{2} \psi_4
+15\binom{d-1}{2} \psi_4
+ 90 \binom{d-1}{3}
\\&\quad= 
\psi_8 +  16(d-1) \psi_6  +  15(d-1) \psi_4^2 + 60(d-1)(d-2) \psi_4
+ 30(d-1)(d-2)(d-3).
\end{align*}
Also
\begin{align*}
&\EE\left( \sum_{i=1}^d Y_i \right)^8
= 
\sum_{i=1}^d \EE Y_i^8 
+ 28\sum_{1\le i < j \le d} \EE Y_i^6 \EE Y_j^2 
+ 28\sum_{1\le j < i \le d} \EE Y_i^6 \EE Y_j^2 
+ 70\sum_{1\le i < j \le d} \EE Y_i^4 \EE Y_j^4 
\\&\hspace{1in}
+ 420\sum_{i < j < k} \EE Y_i^4 \EE Y_j^2 \EE Y_k^2 
+ 420\sum_{j < i < k} \EE Y_i^4 \EE Y_j^2 \EE Y_k^2 
\\&\hspace{1in}
+ 420\sum_{j < k < i} \EE Y_i^4 \EE Y_j^2 \EE Y_k^2 
+ 2520\sum_{i<j<k<l} \EE Y_i ^2\EE Y_j^2 \EE Y_k^2 \EE Y_l^2,
\end{align*}
which is equal to
\begin{align*}
&d\psi_8 + 28\binom{d}{2}\psi_6 + 28\binom{d}{2}\psi_6 + 70\binom{d}{2}\psi_4
\\&\hspace{1in}
+ 420(d-1)\binom{d}{3} \psi_4 + 420(d-1)\binom{d}{3} \psi_4 + 420(d-1)\binom{d}{3} \psi_4
 + 2520\binom{d}{4}
\\&\quad=
 d\psi_8 + 28d(d-1)\psi_6 
+ 35d(d-1)(1+ 12 (d-1) (d-2)) \psi_4
 + 105(d-1) (d-2)(d-3).
\end{align*}
This completes the proof.

\end{proof}

%

\end{document}